\documentclass[letterpaper]{article} % DO NOT CHANGE THIS
\usepackage{aaai2027}  % DO NOT CHANGE THIS
\usepackage[hyphens]{url}  % DO NOT CHANGE THIS
\usepackage{graphicx} % DO NOT CHANGE THIS
\usepackage{natbib}  % DO NOT CHANGE THIS AND DO NOT ADD ANY OPTIONS TO IT
\usepackage{caption} % DO NOT CHANGE THIS AND DO NOT ADD ANY OPTIONS TO IT
\usepackage{rotating}  % for sidewaystable*; not on AAAI's forbidden list
\usepackage{amsmath,amssymb,amsfonts}
\usepackage{amsthm}
\usepackage{algorithmic}
\usepackage{algorithm}
\usepackage{booktabs}
\usepackage{multirow}

\newtheorem{definition}{Definition}
\newtheorem{proposition}{Proposition}
\newtheorem{lemma}{Lemma}
\newtheorem{corollary}{Corollary}
\theoremstyle{remark}
\newtheorem*{remark}{Remark}

\title{SOMtime the World Ain't Fair: \\
Violating Fairness with Self-Organizing Maps}

\author{
    Joseph~Bingham$^{1,*}$, Netanel~Arussy$^{2}$, Dvir~Aran$^{1,2}$
}
\affiliations{
    $^*$ This work is a preprint in AAAI format \\
    $^1$ Faculty of Biology, Technion University \\
    $^2$ Taub Faculty of Computer Science, Technion University \\
    \{jbingham@campus., narussy@campus., dvir.aran@\}technion.ac.il
    
}

\begin{document}

\maketitle

\begin{abstract}
Unsupervised representations are widely assumed to be neutral with respect to sensitive attributes withheld from training. We show this assumption is false. Using SOMtime, a topology-preserving representation based on high-capacity Self-Organizing Maps, we demonstrate that sensitive attributes such as age and income emerge as dominant latent axes in unsupervised embeddings even when explicitly excluded from the input. On two real-world datasets (World Values Survey across five countries, Census-Income (KDD)), SOMtime recovers monotonic orderings aligned with withheld attributes, achieving Spearman correlations up to 0.85 along a single theoretically named axis, whereas PCA, UMAP, t-SNE, autoencoders, and Isomap---evaluated under both per-axis maximum and the more permissive rotation-invariant maximum across all linear directions of their embedding---remain at most 0.44. We characterize this phenomenon as \emph{structured sensitive-attribute leakage}: information not merely present, but \emph{organized} along an interpretable geometric axis. We prove that structured leakage at correlation level $r$ induces a quantitative lower bound on the worst-case group-disparate output of any Lipschitz task acting on the embedding, formalizing the bridge from monotonic ordering to fairness violation. Because the leakage takes the form of a named axis rather than diffuse dependence across coordinates, SOMtime makes structured leakage auditable \emph{before} downstream clustering, recommendation, or visualization tasks are run. Our findings establish that auditing for sensitive-attribute leakage must extend beyond probing classifiers and beyond supervised pipelines to the unsupervised representations that increasingly serve as foundational building blocks of modern ML systems.
\end{abstract}

% Code/data links (kept anonymous for submission)
% \begin{links}
%     \link{Code}{https://anonymous.4open.science/r/SOMtime}
% \end{links}

\section{Introduction}
Modern machine learning pipelines increasingly rely on unsupervised representations as foundational building blocks. Clustering, dimensionality reduction, and learned embeddings are used for customer segmentation, data exploration, visualization, and as preprocessing steps for downstream prediction~\citep{10.1145/3457607}. These representations are often assumed to be neutral, particularly when sensitive attributes such as age, gender, race, or income are deliberately withheld from the input features.

This assumption, known as \emph{fairness through unawareness}~\citep{dwork2012fairness}, has been widely critiqued in the supervised setting: models can reconstruct sensitive attributes from proxy features and reproduce discriminatory patterns even when the protected variable is absent from training~\citep{DBLP:journals/corr/KleinbergMR16,Rabonato_Berton_2024,10.1145/3457607}. The fairness community has focused predominantly on supervised predictors, however, leaving open whether \emph{unsupervised} representations themselves encode sensitive information, and how severely.

%This gap matters: when an unsupervised embedding systematically organizes data along a sensitive axis, any downstream use (clustering, recommendation, visualization, feature extraction) inherits a fairness risk \emph{before any supervised task is even defined}. Collaborative filtering embeddings have been shown to capture user demographics absent demographic features~\citep{ekstrand2018exploring}, fair clustering research documents demographically imbalanced partitions from standard algorithms~\citep{chierichetti2017fair,ghadiri2021socially}, and fair dimensionality reduction work shows t-SNE visualizations dominated by sensitive structure~\citep{peltonen2023fair}.
This gap matters: when an unsupervised embedding organizes data along a sensitive axis, any downstream use (clustering, recommendation, visualization) inherits a fairness risk \emph{before any supervised task is defined}. Collaborative filtering embeddings capture user demographics absent demographic features~\citep{ekstrand2018exploring}, fair clustering documents demographically imbalanced partitions~\citep{chierichetti2017fair,ghadiri2021socially}, and t-SNEs are dominated by sensitive structure~\citep{peltonen2023fair}.

In this work, we make four contributions:

\noindent\textbf{Structured sensitive-attribute leakage.} We define and demonstrate \emph{structured leakage}: information not merely recoverable from an embedding but \emph{organized} along an interpretable geometric axis.

\noindent\textbf{A topology-based auditing method.} We propose SOMtime, based on high-capacity Self-Organizing Maps~\citep{58325}, as an unsupervised auditing tool that exposes structured leakage as a single named axis. SOMtime reveals global sensitive structure---monotonic orderings and gradients---attenuated or invisible in PCA~\citep{Mackiewicz1993-xe}, UMAP~\citep{mcinnes2020umapuniformmanifoldapproximation}, t-SNE~\citep{JMLR:v9:vandermaaten08a}, autoencoders, and Isomap, under both per-axis and rotation-invariant correlation metrics. Isomap is included specifically to test whether the advantage is general to topology preservation or specific to SOMs; we find it is the latter.

\noindent\textbf{Theoretical bridge from ordering to disparity.} We prove that structured leakage at Pearson correlation level $r$ implies an exact group-mean separation of $|r|\sigma_\zeta / \sqrt{p(1-p)}$ along the leaking direction, and a corresponding lower bound on worst-case group-disparate output for any Lipschitz downstream consumer. We further characterize the structural asymmetry between PCA (which suppresses low-variance distributed signal) and SOMs (which preserve Lipschitz-continuous signal regardless of variance).

\noindent\textbf{Implications for downstream fairness.} We argue that auditing for sensitive-attribute leakage must operate at the representation level and include unsupervised components, motivating new auditing strategies that go beyond probing tests.

\begin{figure*}[t!]
  \centering
  \includegraphics[width=\textwidth]{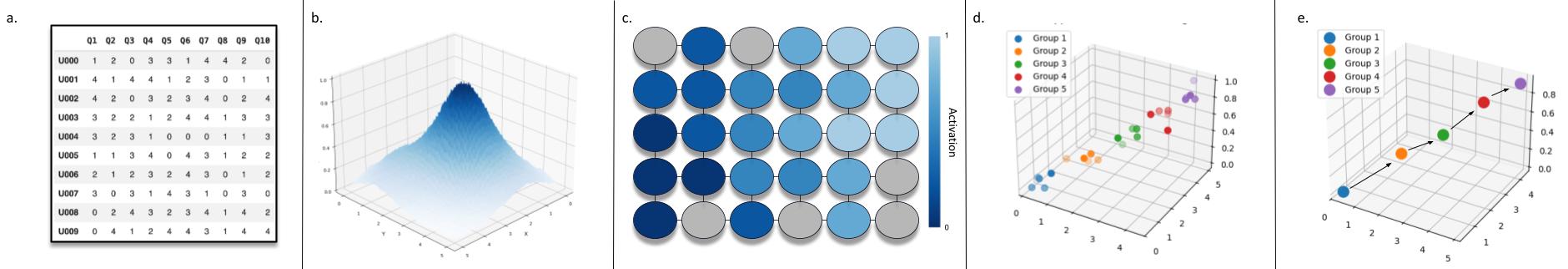}
  \caption{The SOMtime pipeline. (a)~High-dimensional tabular data with sensitive attributes withheld. (b)~A Self-Organizing Map learns a topology-preserving discretization. (c)~Activation patterns across the SOM. (d)~Observations mapped to a 3D embedding (BMU coordinates + activation energy), colored by withheld sensitive attribute (age group). The SOM embedding arranges observations along a monotonic axis aligned with the withheld sensitive attribute, revealing structured information leakage invisible to standard methods. (e)~Trajectory of ordered centroid progressions of protected values recovered.}
  \label{fig:teaser}

\end{figure*}

We study this empirically on two datasets: the World Values Survey (WVS)~\citep{Haerpfer2020-sz,adilazuarda-etal-2025-surveys,NEURIPS2024_9a16935b,zhao2024worldvaluesbenchlargescalebenchmarkdataset} and Census-Income (KDD)~\citep{census-income}. We withhold sensitive attributes entirely from representation learning and measure their emergence as latent structure. SOMtime achieves Spearman correlations up to 0.85 between its activation $z$-axis and withheld attributes, while baselines under both per-axis and rotation-invariant maxima reach at most 0.44.

\section{Background and Related Work}

\subsection*{Auditing Unsupervised Representations}

\emph{Fairness through unawareness}---removing sensitive attributes from inputs to prevent discrimination---is well documented as insufficient~\citep{dwork2012fairness,LIU2026108034}: proxy features enable reconstruction of withheld information and discriminatory patterns~\citep{10.1145/3457607}. The vast majority of fairness criteria (demographic parity, equalized odds, calibration~\citep{DBLP:journals/corr/KleinbergMR16}) are defined with respect to outcome labels and apply to supervised tasks~\citep{Rabonato_Berton_2024}. Whether unsupervised representations are ``fair'', and how to audit fairness in the absence of a label, remains underexplored.

\subsection*{Fairness in Unsupervised Learning}

\emph{Fair clustering} introduces balanced subsets~\citep{chierichetti2017fair} or group-balance constraints~\citep{kleindessner2019guarantees,ahmadian2019clustering,ghadiri2021socially} to control demographic composition of clusters~\citep{chhabra2021overview}. \emph{Fair representation learning} learns latents predictive of useful features but uninformative of protected attributes~\citep{zemel2013learning}, often via adversarial training~\citep{iwasawa2018censoring,madras2018learning}, MMD matching~\citep{louizos2016variational}, or concept erasure~\citep{ravfogel2020null}; standard autoencoders are also widely used as general-purpose nonlinear baselines~\citep{Hinton2006-uo}. \emph{Fair dimensionality reduction} includes the ``price of fair PCA''~\citep{samadi2018price}, fairness-aware nonlinear embeddings~\citep{peltonen2023fair}, and conditional t-SNE that factors out specified structure~\citep{kang2021conditional}. \emph{Auditing for sensitive-attribute leakage} commonly uses probing classifiers~\citep{10.1145/3457607}, primarily applied to supervised embeddings; mutual information estimation~\citep{moyer2018invariant} provides a more general measure of dependence but is demanding and less interpretable for continuous ordinal attributes.

\subsection*{Topology-Preserving Representations and the Auditing Gap}

Self-Organizing Maps (SOMs)~\citep{58325} learn topology-preserving discretizations of high-dimensional data onto a lattice of prototype vectors. While SOMs are widely used for clustering and visualization, they have seen little treatment in recent fairness literature, which has focused on deep learning and differentiable embeddings.

Critically, prior work on auditing fairness in unsupervised learning focuses on \emph{local separability} (whether groups are locally distinguished) or \emph{classification leakage} (whether a probe can predict group membership). We identify a distinct underexamined form of leakage: \textbf{global geometric ordering}. When an unsupervised representation arranges observations along a monotonic axis aligned with $s$, this is qualitatively different from local separability and from latent extractability: standard low-dimensional embeddings (PCA, UMAP) attenuate this signal, while high-capacity SOMs amplify it, making it visible to inspection rather than only to a probe.

\noindent\textit{Direction-Finding Audits.} Concept activation vectors, linear probes, and mutual information estimators (e.g., MINE) search for directions or scalar statistics capturing attribute leakage. SOMtime differs in two respects: it requires no parametric probe (the leaking direction is identified by post-hoc correlation, not gradient fitting), and the geometric signature it produces is interpretable in a way that scalar MI estimates are not. We view these methods as complementary and direct comparison to MINE-style audits in continuous ordinal settings as a natural extension.

\noindent\textit{Other Manifold Methods.} A wider family preserves manifold structure, including Isomap, Laplacian eigenmaps, diffusion maps, neural gas, and growing SOMs. We include Isomap as a baseline to test whether SOMtime's advantage stems from topology preservation generally or from SOMs specifically; comparison against the others is left to future work.

\section{Method: SOMtime as an Auditing Tool}

We position SOMtime not as a solution to fairness, but as a \emph{lens} that makes sensitive structure in unsupervised representations explicit.

\subsection*{Problem Setup}
We consider a dataset-agnostic unsupervised setting. Each observation is a fixed-length feature vector $\mathbf{x}_i \in \mathbb{R}^d$, optionally associated with a sensitive attribute $s_i$ that is \textbf{withheld} from all representation learning. All representation learning steps are fully unsupervised with respect to $s$; sensitive attributes are used \textbf{only for post-hoc auditing}.

\subsection*{Self-Organizing Map Representation}
We map the input data into a two-dimensional discrete latent space using a Self-Organizing Map (SOM). Let $X \in \mathbb{R}^{N \times d}$ denote the input matrix. The SOM consists of a square lattice of $K \times K$ units, where each unit $j$ is associated with a prototype vector $\mathbf{w}_j \in \mathbb{R}^d$.

For each input observation $\mathbf{x}_i$, the best-matching unit (BMU) is defined as

$j^* = \arg\min_j \| \mathbf{x}_i - \mathbf{w}_j \|.$
Prototype vectors are updated according to
\[
\mathbf{w}_j(t+1) = \mathbf{w}_j(t) + \alpha(t)\, h(j, j^*, t)\, (\mathbf{x}_i - \mathbf{w}_j(t)),
\]
where $\alpha(t)$ is a learning rate and $h(\cdot)$ is a Gaussian neighborhood function defined over the lattice topology. In contrast to conventional SOM usage for two-dimensional visualization, we employ a map with a large number of neurons ($K = 5 \cdot N^{0.54}$), following a slight modification of the heuristic proposed by \citet{vesanto2000clustering}. We verify in the supplementary material that SOMtime's recovery of the withheld sensitive attribute is stable across a $5\times$ range of lattice-size scalings and that the chosen learning rate ($0.75$) sits at a clear local optimum.

\subsection*{3D Embedding and Axis Recovery}
To enable continuous analysis, we transform the discrete SOM grid into 3D: each observation inherits the $(x, y)$ coordinates of its BMU and a third coordinate $z_i = \|\mathbf{x}_i - \mathbf{w}_{j^*}\|$, the per-observation quantization error. The resulting embedding $\mathbf{z}_i \in \mathbb{R}^3$ preserves local neighborhood relations while enabling global geometric analysis.

To assess whether the representation encodes a global ordering aligned with $s$, we analyze the geometry of SOM unit centroids. Cluster centers from the SOM distance map are computed in XY-space, then mapped to centroid activations on the distance map. We then construct a trajectory adjacency matrix on these centroids via a greedy nearest-neighbour procedure that enforces monotonic progression along the axis (in supplementary). A representation exhibits \textit{global geometric ordering} with respect to $s$ if there is a path along which $s$ varies monotonically---distinct from \textit{local separability}, which measures if observations with similar $s$ are nearby. We quantify ordering strength via the Spearman correlation $\rho$ between $s$ and position along the recovered path.

\subsection*{Auditing via Correlation Analysis}
We assess alignment between $s$ and the learned representation via two complementary metrics, both reported in Table~\ref{tab:leakage}. The first is the \emph{per-axis maximum} correlation between $s$ and any single coordinate of the embedding (the metric used in the original conference version). The second is the \emph{rotation-invariant maximum}: the multiple correlation $\sqrt{R^2}$ from cross-validated ridge regression of $s$ on the full embedding, equivalent to the maximum Pearson correlation along any unit direction in the embedding space; the Spearman analogue rank-transforms both sides before regression. We use 5-fold cross-validation; a CV $R^2 \le 0$ is reported as $0$, indicating no linear projection generalizes better than the constant. The rotation-invariant metric addresses the concern that per-axis correlations underestimate leakage in rotation-equivariant embeddings (UMAP, t-SNE, AE) whose coordinate axes are arbitrary up to rotation.

%For SOMtime we report only the per-axis correlation along the activation $z$-axis---the direction predicted by our theoretical analysis to carry structured leakage---in both metric columns. This is strictly stricter than the rotation-invariant baseline metric: SOMtime is committed in advance to one named direction, baselines are evaluated over all linear combinations. We adopt this asymmetry because it isolates the auditing signal from any post-hoc direction selection. The evaluation is \emph{purely diagnostic}: $s$ is never used during training, embedding construction, or ordering extraction.

For SOMtime we report only the correlation along the activation $z$-axis, in both metric columns. This is strictly stricter than the rotation-invariant baseline metric: SOMtime is committed in advance to one named direction while baselines are evaluated over all linear combinations. The evaluation is \emph{purely diagnostic}: $s$ is never used during training, embedding construction, or ordering extraction.

\section{Theoretical Analysis}
\label{sec:theory}

\subsection*{Defining Structured Leakage}

Let $E: \mathcal{X} \to \mathcal{Z} \subseteq \mathbb{R}^{d_z}$ be an embedding and $s: \mathcal{X} \to \mathcal{S}$ a sensitive attribute withheld from $E$. For a unit vector $\hat{\mathbf{v}} \in \mathbb{S}^{d_z-1}$, write $\zeta_{\hat{\mathbf{v}}}(\mathbf{x}) := \langle \hat{\mathbf{v}}, E(\mathbf{x}) \rangle$.

\begin{definition}[Structured leakage]
\label{def:structured}
An embedding $E$ exhibits \textbf{structured leakage} with respect to $s$ at correlation level $r \in [0, 1]$ if there exists a unit direction $\hat{\mathbf{v}}$ such that the absolute Pearson correlation $|\rho_p(\zeta_{\hat{\mathbf{v}}}(X), s)| \geq r$. The Spearman analogue $|\rho_s(\zeta_{\hat{\mathbf{v}}}(X), s)| \geq r$ provides a rank-based version, which coincides with the Pearson version when $\zeta_{\hat{\mathbf{v}}}$ and $s$ are related by a monotonic transformation.
\end{definition}

\begin{remark}
Structured leakage is distinct from \emph{probe-recoverable leakage}, the condition that some classifier $h: \mathcal{Z} \to \mathcal{S}$ predicts $s$ from $E(X)$ with above-chance accuracy. An embedding may be probe-recoverable without being structurally leaking, because probe-recoverability admits arbitrarily complex (e.g., nonlinear or multi-axis) decoders, whereas structured leakage requires that $s$ be encoded along a single linear direction. Our autoencoder ablation (supplemental material) is consistent with this distinction.
\end{remark}

An \emph{empirical signature} of structured leakage at the cluster level is that the trajectory recovery procedure (supplementary), applied to BMU centroids of the embedding, produces a trajectory whose ordering matches the natural ordering of $s$ (Table~\ref{tab:top}).

\subsection*{Structured Leakage Implies Downstream Disparity}

\begin{proposition}[Disparity bound under structured leakage]
\label{prop:disparity}
Let $E$ exhibit structured leakage with respect to $s$ at Pearson correlation level $r$ along direction $\hat{\mathbf{v}}$, with $\zeta(\mathbf{x}) := \langle \hat{\mathbf{v}}, E(\mathbf{x}) \rangle$. Suppose $s$ takes binary values $\{0, 1\}$ with $\mathbb{P}(s = 1) = p \in (0, 1)$.

\noindent\textbf{(a)} The group-conditional means of $\zeta$ satisfy
\[
\bigl|\, \mathbb{E}[\zeta(X) \mid s = 1] - \mathbb{E}[\zeta(X) \mid s = 0] \,\bigr|
\;=\; \frac{|r|\, \sigma_\zeta}{\sqrt{p(1-p)}},
\]
where $\sigma_\zeta := \sqrt{\mathrm{Var}(\zeta(X))}$. This is an exact equality.

\noindent\textbf{(b)} The linear functional $g^*(\mathbf{z}) := \langle \hat{\mathbf{v}}, \mathbf{z} \rangle$ is $1$-Lipschitz and produces group-disparate output equal to $|r|\sigma_\zeta / \sqrt{p(1-p)}$.

\noindent\textbf{(c)} The worst-case group-mean disparity over $L$-Lipschitz functions $g: \mathcal{Z} \to \mathbb{R}$ is bounded below by the disparity achieved by the leaking-direction functional:
\[
\sup_{g\ L\text{-Lip}} \!\bigl|\, \mathbb{E}[g(E(X)) | s {=} 1] - \mathbb{E}[g(E(X)) | s {=} 0] \,\bigr|
\;\geq\; \frac{L \cdot |r|\, \sigma_\zeta}{\sqrt{p(1-p)}},
\]
attained by $L \cdot g^*$.
\end{proposition}

\begin{proof}
Let $\mu_k := \mathbb{E}[\zeta(X) \mid s = k]$. Direct computation gives $\mathrm{Cov}(\zeta, s) = p(1-p)(\mu_1 - \mu_0)$ and $\mathrm{Var}(s) = p(1-p)$. Substituting into the Pearson definition yields $|\mu_1 - \mu_0| = |r|\sigma_\zeta/\sqrt{p(1-p)}$, giving \textbf{(a)}. For \textbf{(b)}, $g^*(E(\mathbf{x})) = \zeta(\mathbf{x})$ and Cauchy--Schwarz gives the $1$-Lipschitz property. For \textbf{(c)}, the function $L \cdot g^*$ is $L$-Lipschitz, and applying \textbf{(a)} to it gives group-mean disparity equal to $L \cdot |r|\sigma_\zeta / \sqrt{p(1-p)}$; the supremum over all $L$-Lipschitz $g$ is at least as large.
\end{proof}

%\begin{remark}[Relationship to the Wasserstein bound]
%The universal upper bound on $L$-Lipschitz group-mean disparity is given by Kantorovich--Rubinstein duality as $L \cdot W_1(P_{E(X) | s=1}, P_{E(X) | s=0})$ in $\mathcal{Z}$. The quantity $|r|\sigma_\zeta / \sqrt{p(1-p)}$ is itself a lower bound on this Wasserstein distance, since the difference of conditional means along any direction is bounded above by the conditional $W_1$ distance. Part \textbf{(c)} therefore establishes that the worst-case Lipschitz-induced disparity is necessarily nonzero whenever structured leakage is present; tightening this lower bound toward the Wasserstein quantity requires distributional information beyond the Pearson correlation.
%\end{remark}

\begin{remark}[Relationship to the Wasserstein bound]
By Kantorovich--Rubinstein duality, the universal upper bound on $L$-Lipschitz group-mean disparity is $L \cdot W_1(P_{E(X)|s=1}, P_{E(X)|s=0})$; the quantity $|r|\sigma_\zeta/\sqrt{p(1-p)}$ is a lower bound on this Wasserstein distance. Part \textbf{(c)} therefore establishes that worst-case Lipschitz-induced disparity is nonzero whenever structured leakage is present; tightening toward the Wasserstein quantity requires distributional information beyond Pearson correlation.
\end{remark}

\begin{remark}
The result extends to ordinal $s$ with $K$ values by applying \textbf{(a)} pairwise. The Spearman version follows by applying the same argument to rank-transformed variables.
\end{remark}

\begin{corollary}[Demographic parity violation]
\label{cor:dp}
Under the conditions of Proposition~\ref{prop:disparity}, let $C: \mathcal{Z} \to \{c_1, \ldots, c_M\}$ be any partition obtained by thresholding $\zeta$. Then for at least one cluster $c_m$, $|\mathbb{P}(C = c_m \mid s = 1) - \mathbb{P}(C = c_m \mid s = 0)| > 0$ whenever $|r| > 0$. The magnitude of the worst-case cluster-level disparity is bounded below by a strictly increasing function of $|r|$.
\end{corollary}

Proposition~\ref{prop:disparity} provides the bridge from monotonic ordering to fairness violation that grounds the rest of the paper. The empirical Pearson correlations reported in Table~\ref{tab:leakage} (up to $0.80$ for WVS Canada) translate via part \textbf{(a)} into a guaranteed group-mean separation of approximately $1.6\sigma_\zeta$ for balanced binary subgroups, and via part \textbf{(c)} into a guaranteed lower bound on the worst-case Lipschitz-induced downstream disparity. The result also formalizes the generalization of the critique of \emph{fairness through unawareness}~\citep{10.1145/3457607} from supervised to unsupervised settings.

\subsection*{Mechanism: Variance Bias of PCA, Locality Preservation of SOMs}

\begin{lemma}[PCA variance bias]
\label{lem:pca}
Let $\mathbf{X} \in \mathbb{R}^d$ be centered with covariance $\Sigma = \mathbf{V}\Lambda \mathbf{V}^\top$, $\lambda_1 \geq \cdots \geq \lambda_d \geq 0$. Suppose $s = \boldsymbol{\beta}^\top \mathbf{X} + \varepsilon$ with $\varepsilon \perp \mathbf{X}$. Let $\gamma_i := \langle \mathbf{v}_i, \boldsymbol{\beta} \rangle$. Then $|\rho_i| = |\gamma_i| \sqrt{\lambda_i} / \sigma_s$, and PCA truncation to the top $K$ components captures squared correlation $\sum_{i \leq K} \gamma_i^2 \lambda_i / \sigma_s^2$, suppressing signal in components $i > K$.
\end{lemma}

\begin{proof}
$\mathrm{Cov}(s, \langle \mathbf{v}_i, \mathbf{X} \rangle) = \boldsymbol{\beta}^\top \Sigma \mathbf{v}_i = \gamma_i \lambda_i$, and $\mathrm{Var}(\langle \mathbf{v}_i, \mathbf{X} \rangle) = \lambda_i$. Direct substitution gives this form.
\end{proof}

The crucial consequence is that the eigenvalue ordering has, in general, no relationship to the coefficients $|\gamma_i|$ that determine where the $s$-signal lives. In real tabular data where $s$ correlates weakly with many features (Table~\ref{tab:correlation}), the coefficients $\gamma_i$ are distributed across many eigenvectors. Truncating to the top $K$ principal components discards exactly the signal we wish to audit for. UMAP and t-SNE impose related low-dimensional bottlenecks, so the same suppression argument applies.

\begin{definition}[Topology-preserving SOM]
\label{def:topology}
A trained $K \times K$ SOM with prototypes $\{\mathbf{w}_i\}$ is $\epsilon$-topology-preserving on data $\mathbf{X}$ if for any two lattice-adjacent units $i, j$: $\|\mathbf{w}_i - \mathbf{w}_j\| \leq \epsilon$.
\end{definition}

\begin{proposition}[Preservation of Lipschitz sensitive attributes]
\label{prop:som}
Suppose $\mathbf{X}$ lies on a manifold $\mathcal{M}$, and $s: \mathcal{M} \to \mathbb{R}$ is $L_s$-Lipschitz under the Euclidean metric. If a trained SOM is $\epsilon$-topology-preserving (Definition~\ref{def:topology}) and its prototypes lie on or near $\mathcal{M}$, then for any lattice-adjacent prototypes:
$
|s(\mathbf{w}_i) - s(\mathbf{w}_j)| \;\leq\; L_s\, \epsilon.
$
\end{proposition}

\begin{proof}
$|s(\mathbf{w}_i) - s(\mathbf{w}_j)| \leq L_s \|\mathbf{w}_i - \mathbf{w}_j\| \leq L_s \epsilon$.
\end{proof}

\begin{remark}[Status of topology preservation]
Proposition~\ref{prop:som} is conditional on the SOM achieving topology preservation at lattice resolution $\epsilon$. For 1D SOMs, convergence to a topology-preserving configuration is proven under standard conditions~\citep{58325}. For 2D SOMs, a fully general convergence proof remains open, but topology preservation is empirically robust under standard Kohonen training when the lattice resolution exceeds the intrinsic dimension of the data manifold~\citep{vesanto2000clustering}. The scaling $K = 5 \cdot N^{0.54}$ used here follows this empirical guidance and yields lattices on which topology preservation is verifiable post-hoc via the per-unit quantization error.
\end{remark}

\begin{remark}[Local preservation vs.\ global emergence]
Proposition~\ref{prop:som} establishes that lattice-adjacent prototypes have $s$-values differing by at most $L_s \epsilon$, which is a \emph{local} preservation guarantee on the SOM's discrete grid. It does not by itself prove that a \emph{single global monotone axis} on the lattice must emerge with respect to $s$. The emergence of such a global axis depends additionally on the topology of $s$ along the underlying manifold: when $s$ is well-described by a single monotonic function of a one-dimensional intrinsic coordinate (as is the case for age along the ethical-development gradient that organizes survey responses), local preservation chains into global ordering through the lattice. When $s$ depends on multiple independent intrinsic coordinates, local preservation may yield a global \emph{gradient field} rather than a single ordered axis. A complete geometric characterization of when local preservation implies global monotone ordering is beyond the scope of this paper; our empirical results (Section V) demonstrate that the single-axis case captures the WVS and Census attributes we study.
\end{remark}
\begin{remark}[Local preservation vs.\ global emergence]
Proposition~\ref{prop:som} gives a \emph{local} preservation guarantee on lattice-adjacent prototypes; it does not by itself prove that a single global monotone axis must emerge. Emergence depends on the topology of $s$ along the manifold: when $s$ is described by a single monotonic function of a one-dimensional intrinsic coordinate (as with age along the ethical-development gradient that organizes survey responses), local preservation chains into global ordering. When $s$ depends on multiple intrinsic coordinates, a gradient field may emerge instead. Our empirical results demonstrate that the single-axis case captures the WVS and Census attributes we study.
\end{remark}

The contrast between Lemma~\ref{lem:pca} and Proposition~\ref{prop:som} is the structural asymmetry the paper exploits. PCA's recovery of $s$ depends on alignment between the eigenvalue ordering of $\Sigma$ and the coefficients $\boldsymbol{\beta}$, and is suppressed when $s$ is carried by low-variance directions. The SOM's preservation of $s$ depends only on the Lipschitz continuity of $s$ along the manifold, a property of $s$ rather than of the data covariance.

\section{Experimental Protocol}
\label{sec:protocol}

\subsection*{Implementation Details}
All experiments are implemented in Python using standard scientific computing libraries. SOM hyperparameters are fixed across all experiments and not tuned using sensitive attribute labels: $\sigma = 0.7$, learning rate $= 0.75$, with lattice dimension $K = 5 \cdot N^{0.54}$. A question is retained only if it has no missing values. Since SOM, t-SNE, and UMAP all involve stochastic initialization, reported results are averaged over five independent runs; standard deviations were below 0.03 for all Spearman correlations.
SOM hyperparameters are fixed across all experiments and not tuned using sensitive labels: $\sigma = 0.7$, learning rate $= 0.75$, $K = 5 \cdot N^{0.54}$. Questions with any missing values are dropped. SOM, t-SNE, and UMAP results are averaged over five runs; $\sigma$ below $0.03$ for all Spearman correlations.

\subsection*{Datasets}

\noindent\textbf{World Values Survey (WVS).} The WVS~\citep{Haerpfer2020-sz} is a globally representative dataset capturing political, cultural, and moral values, recently used in top-tier ML venues~\citep{adilazuarda-etal-2025-surveys,NEURIPS2024_9a16935b,zhao2024worldvaluesbenchlargescalebenchmarkdataset}. We use data from five countries: Canada ($N\!=\!4{,}018$), Romania ($N\!=\!3{,}200$), Germany ($N\!=\!1{,}528$), China ($N\!=\!3{,}036$), and USA ($N\!=\!2{,}609$), with \textbf{age} as the withheld sensitive attribute. We use a subset of 22 questions on ethics and morals. Table~\ref{tab:correlation} reports mean correlations.

\begin{table}[t!]
\small
\centering
\caption{Correlation of questions from WVS and Census (KDD) with the sensitive attribute of age. Max and mean correlations across used questions are given, alongside the number of responses.}
\label{tab:correlation}
\begin{tabular}{llccc}
\toprule
\textbf{Dataset} & \textbf{Domain} & \textbf{Max} & \textbf{Mean} & \textbf{Responses} \\
\midrule
\multirow{5}{*}{WVS}
& Canada  & -0.35 & -0.09 & 4,018 \\
& Romania & -0.15 & 0.03  & 3,200 \\
& Germany & -0.34 & -0.07 & 1,528 \\
& China   & -0.25 & -0.05 & 3,036 \\
& USA     & -0.26 & -0.07 & 2,609 \\
\midrule
\multirow{3}{*}{Census (KDD)}
& Age            & 0.27 & -0.05 & \multirow{3}{*}{27,024} \\
& Income         & 0.12 & -0.04 \\
& Capital Gains  & 0.10 & 0.07 \\
\bottomrule
\end{tabular}

\end{table}

\noindent\textbf{Census-Income (KDD).} The Census-Income dataset~\citep{census-income} is a standard UCI benchmark containing 299,285 records with 40 demographic and economic features from the 1994 and 1995 Current Population Surveys. We treat age, income, and capital gains as withheld sensitive attributes.

We applied three preprocessing steps before representation learning. First, we removed columns without a clean numerical or ordinal encoding, since baselines do not handle arbitrary categoricals uniformly. Second, for each withheld attribute $s$, we removed columns whose absolute Pearson correlation with $s$ exceeded $0.3$ on the encoded data---a threshold chosen above the maximum surviving correlation in WVS ($|\rho| \le 0.35$; Table~\ref{tab:correlation}) so the two datasets place comparable demands on embedding methods. After filtering, 25 of the original 40 columns remained, the strongest surviving correlation against age being $|\rho| = 0.27$. Third, we dropped records with missing values; WVS missingness has consistent semantics representable as $-1$, but Census missingness is ambiguous. The resulting 27{,}024 observations yield statistically significant $p$-values for all reported correlations.

The proxy-removal step warrants explanation because it can be misread as curating the data in our favour. The opposite is the case. Columns most strongly correlated with $s$ are exactly the columns through which any unsupervised method, PCA included, would recover $s$ trivially; removing them makes the recovery problem harder, and harder uniformly across methods. What remains is a regime in which no single feature is a useful indicator of the sensitive attribute---the strongest surviving correlation against age is $0.27$, and the mean across features is near zero. A method that recovers $s$ from this data does so by aggregating distributed, individually weak signals into coherent geometric structure---precisely the threat model that separates structured leakage from probe-recoverable leakage. It is also the realistic threat model in practice: embeddings that get audited at all will typically have had their strongest proxies removed by responsible data hygiene before any auditing tool sees them. The sensitive attribute is used only at this preprocessing stage and at post-hoc validation, never during representation learning, hyperparameter selection, or trajectory recovery. WVS required no analogous filtering: no single WVS feature exceeds $|\rho| = 0.35$ with age (Table~\ref{tab:correlation}), so the low-proxy regime that the Census preprocessing produces is WVS's default.

Proxy removal makes the recovery problem harder, uniformly across methods: it eliminates columns through which any unsupervised method would recover $s$ trivially. What remains is a regime in which no single feature is a useful indicator (strongest surviving correlation against age: $0.27$). A method that recovers $s$ from this data must aggregate distributed weak signals into coherent geometric structure---precisely the threat model that separates structured leakage from probe-recoverable leakage, and the realistic threat model in practice, since audited embeddings will typically have had their strongest proxies removed by responsible data hygiene. The sensitive attribute is used only at this preprocessing stage and at post-hoc validation, never during representation learning or trajectory recovery. WVS required no analogous filtering; no features exceeds $|\rho| = 0.35$ with age.

\begin{table*}[t!]
\caption{Sensitive attribute leakage. \emph{Per-axis} columns: max absolute correlation between any single embedding coordinate and the withheld $s$, best across dimensionalities $d \in \{2, \ldots, 50\}$. \emph{Rotation-invariant} columns (-I): multiple correlation $\sqrt{R^2}$ from cross-validated ridge regression of $s$ on the full embedding (max correlation along any unit direction). SOMtime uses only its $z$-axis---a stricter protocol than all baselines. Zero -I entries indicate cross-validated $R^2 \le 0$. ISO country codes. All non-zero per-axis correlations have $p < 10^{-3}$ except UMAP on Census Gains ($\rho = 0.002$, $p = 0.89$). $\sigma$ across 5 runs $< 0.03$.}
\label{tab:leakage}
\centering
\scriptsize
\setlength{\tabcolsep}{1.8pt}
\begin{tabular}{ll cc cc cc cc cc cc cc cc cc cc cc}
\toprule
\multirow{2}{*}{\textbf{Dataset}}
& \multirow{2}{*}{\textbf{Domain}}
& \multicolumn{2}{c}{\textbf{PCA}}
& \multicolumn{2}{c}{\textbf{PCA-I}}
& \multicolumn{2}{c}{\textbf{UMAP}}
& \multicolumn{2}{c}{\textbf{UMAP-I}}
& \multicolumn{2}{c}{\textbf{t-SNE}}
& \multicolumn{2}{c}{\textbf{t-SNE-I}}
& \multicolumn{2}{c}{\textbf{AE}}
& \multicolumn{2}{c}{\textbf{AE-I}}
& \multicolumn{2}{c}{\textbf{Isomap}}
& \multicolumn{2}{c}{\textbf{Isomap-I}}
& \multicolumn{2}{c}{\textbf{SOMtime}} \\
\cmidrule(lr){3-4} \cmidrule(lr){5-6} \cmidrule(lr){7-8} \cmidrule(lr){9-10} \cmidrule(lr){11-12} \cmidrule(lr){13-14} \cmidrule(lr){15-16} \cmidrule(lr){17-18} \cmidrule(lr){19-20} \cmidrule(lr){21-22} \cmidrule(lr){23-24}
& & P. & S. & P. & S. & P. & S. & P. & S. & P. & S. & P. & S. & P. & S. & P. & S. & P. & S. & P. & S. & P. & S. \\
\midrule
\multirow{5}{*}{WVS}
& CA & 0.25 & 0.22 & 0.41 & 0.44 & 0.31 & 0.31 & 0.43 & 0.42 & 0.22 & 0.21 & 0.28 & 0.27 & 0.33 & 0.34 & 0.41 & 0.42 & 0.24 & 0.25 & 0.42 & 0.43 & \textbf{0.75} & \textbf{0.85} \\
& RO & 0.06 & 0.04 & 0.19 & 0.17 & 0.06 & 0.05 & 0.00 & 0.00 & 0.06 & 0.05 & 0.00 & 0.00 & 0.16 & 0.15 & 0.06 & 0.00 & 0.06 & 0.08 & 0.00 & 0.00 & \textbf{0.66} & \textbf{0.59} \\
& DE & 0.16 & 0.18 & 0.40 & 0.42 & 0.17 & 0.19 & 0.31 & 0.30 & 0.19 & 0.35 & 0.17 & 0.19 & 0.29 & 0.29 & 0.38 & 0.37 & 0.19 & 0.21 & 0.37 & 0.37 & \textbf{0.79} & \textbf{0.73} \\
& CN & 0.22 & 0.22 & 0.30 & 0.30 & 0.21 & 0.20 & 0.24 & 0.22 & 0.24 & 0.23 & 0.21 & 0.20 & 0.19 & 0.25 & 0.26 & 0.24 & 0.23 & 0.22 & 0.25 & 0.25 & \textbf{0.80} & \textbf{0.58} \\
& US & 0.20 & 0.23 & 0.24 & 0.29 & 0.33 & 0.34 & 0.18 & 0.23 & 0.25 & 0.24 & 0.00 & 0.00 & 0.27 & 0.29 & 0.13 & 0.22 & 0.26 & 0.26 & 0.19 & 0.23 & \textbf{0.80} & \textbf{0.52} \\
\midrule
\multirow{3}{*}{\shortstack{Census\\(KDD)}}
& Age    & 0.17 & 0.11 & 0.28 & 0.36 & 0.08 & 0.09 & 0.00 & 0.00 & 0.08 & 0.10 & 0.00 & 0.04 & 0.20 & 0.22 & 0.03 & 0.09 & 0.08 & 0.10 & 0.00 & 0.07 & \textbf{0.50} & \textbf{0.83} \\
& Income & 0.31 & 0.21 & 0.32 & 0.28 & 0.07 & 0.07 & 0.00 & 0.00 & 0.08 & 0.08 & 0.00 & 0.00 & 0.22 & 0.25 & 0.21 & 0.00 & 0.08 & 0.08 & 0.06 & 0.05 & \textbf{0.48} & \textbf{0.69} \\
& Gains  & 0.23 & 0.11 & 0.00 & 0.08 & 0.01 & 0.00 & 0.00 & 0.00 & 0.02 & 0.04 & 0.00 & 0.00 & 0.35 & 0.40 & 0.00 & 0.00 & 0.04 & 0.06 & 0.00 & 0.00 & \textbf{0.34} & \textbf{0.43} \\
\bottomrule
\end{tabular}

\end{table*}

\subsection*{Baselines}

We compare against five widely used unsupervised dimensionality reduction techniques. \textbf{PCA:} linear embedding via principal component analysis. \textbf{UMAP:} nonlinear embedding with default parameters. \textbf{t-SNE:} nonlinear embedding optimized for local neighborhood structure. \textbf{Isomap:} topology-preserving manifold learner that embeds data via classical MDS on geodesic distances over a $k$-NN graph; included specifically to test whether SOMtime's auditing advantage stems from topology preservation in general or from the SOM's competitive-learning dynamics in particular. \textbf{Autoencoder (AE):} nonlinear encoder--decoder networks at three scales (small $\approx\!7K$, medium $\approx\!95K$, large $\approx\!1.4M$ parameters), bottleneck dimensionalities $\in\{3, 10, 22\}$; the large configuration is capacity-matched to the SOM.

For PCA, UMAP, t-SNE, and Isomap we report the best correlation across the full dimensionality sweep $d \in \{2, 3, \ldots, 50\}$, which lets each baseline ``choose'' the dimensionality (including standard 2D/3D visualization) most likely to expose $s$. Max-over-$d$ values cluster within roughly 2\% across the range for all methods, so the choice of $d$ does not materially affect results. We also performed full hyperparameter ablations within each method and selected best values per dataset.

A probing classifier requires labeled sensitive attributes and is therefore a supervised auditing tool. SOMtime, by contrast, reveals sensitive structure as an \textbf{emergent geometric property of the representation}. For the leakage evaluation (Experiment 1, Table~\ref{tab:leakage}) we apply both protocols described above: per-axis max (columns PCA, UMAP, t-SNE, AE, Isomap) reports the original conference comparison; rotation-invariant (-I columns) reports the maximum correlation across any linear direction, removing axis-orientation bias against rotation-equivariant baselines. For both, SOMtime is evaluated along its single named $z$-axis. For Experiment 2 (global ordering) we extract a single 1D summary axis---the strongest principal component for baselines, the graph-diameter path along the SOM lattice for SOMtime---and report its Spearman correlation with $s$.

\subsection*{Experiments}

\noindent\textbf{Experiment 1: Sensitive Attribute Recoverability.} For each method, we compute the embedding with $s$ withheld, then measure the max correlation between $s$ and a linear projection of the embedding under both protocols (Table~\ref{tab:leakage}).

\noindent\textbf{Experiment 2: Global Ordering Emergence.} We extract the dominant 1D latent axis from each representation and compute its Spearman correlation with the sensitive attribute.

\noindent\textbf{Experiment 3: Recovering Sensitive Topology from 3D Embedding.} We construct full trajectory-based ordering of the sensitive attribute of age in the 3D embedding, beyond the known age groupings.

\section{Results}

\subsection*{Sensitive Attribute Leakage}

Table~\ref{tab:leakage} reports per-axis correlations between each embedding and the withheld sensitive attribute, the metric used in the original conference version. SOMtime's correlation along its activation $z$-axis exceeds the per-axis maximum of any baseline on every (dataset, attribute) pair. On WVS, SOMtime attains Spearman correlations of 0.52--0.85 with age across five countries, while PCA, UMAP, t-SNE, and AE reach at most 0.34. On Census, SOMtime achieves Spearman correlations of 0.83 (age), 0.69 (income), and 0.43 (gains), against baseline maxima of 0.22, 0.25, and 0.40 respectively.

\noindent\textbf{Rotation invariance does not close the gap.} Under the rotation-invariant metric (Table~\ref{tab:leakage}, -I columns), every baseline number stays the same or increases, since the metric is strictly more permissive than per-axis maximum. SOMtime's lead is preserved on every dataset. The strongest baseline (PCA on most tasks) reaches at most 0.44 (WVS Canada Spearman), still well below SOMtime's 0.85, and the factor-of-difference between SOMtime and the strongest baseline ranges from $1.7\times$ to $5.4\times$ across all eight (dataset, attribute) pairs. PCA benefits most, consistent with Lemma~\ref{lem:pca}: PCA's coordinates are linear projections of the input, so any input-linear signal in $s$ is preserved, just distributed across components. UMAP, t-SNE, and AE benefit less or not at all---their nonlinear bottlenecks destroy the linear-projection signal the metric is designed to detect. Several baseline cells read $0.00$, corresponding to cross-validated $R^2 \le 0$: WVS Romania, Census Age, Census Income, and Census Capital Gains all contain baselines that fail this generalization test entirely, while SOMtime's $z$-axis Spearman on the same data is 0.59, 0.83, 0.69, and 0.43 respectively.

\noindent\textbf{The gap is largest where the signal is weakest.} The ratio of SOMtime's Spearman to the strongest baseline's rotation-invariant Spearman grows as the per-feature signal weakens. WVS Canada has the strongest per-feature correlations ($|\rho|$ up to 0.35 in Table~\ref{tab:correlation}) and shows SOMtime's smallest relative gap ($1.9\times$, 0.85 vs.\ 0.44). Gains has the weakest per-feature correlations ($|\rho|$ up to 0.10) and the largest relative gap ($5.4\times$, 0.43 vs.\ 0.08). This pattern is consistent with the mechanism formalized by Lemma~\ref{lem:pca} and Proposition~\ref{prop:som}: when $s$ is carried by a strong, near-linear signal, even projection-based methods can recover it; when $s$ is distributed across many features with weak per-feature correlation, the SOM's locality-based preservation is the only mechanism that consistently aggregates the distributed signal into a recoverable axis. By Proposition~\ref{prop:disparity}(c), the correlations in Table~\ref{tab:leakage} translate directly into a lower bound of $L \cdot |r|\sigma_\zeta / \sqrt{p(1-p)}$ on the worst-case group-disparate output of any $L$-Lipschitz downstream consumer of the embedding.

\noindent\textbf{The gap is largest where the signal is weakest.} SOMtime's advantage over the strongest baseline grows as per-feature signal weakens: WVS Canada (max $|\rho| = 0.35$) has the smallest relative gap ($1.9\times$); Census Gains (max $|\rho| = 0.10$) has the largest ($5.4\times$). This is exactly the mechanism of Lemma~\ref{lem:pca} and Proposition~\ref{prop:som}: strong linear signals are recoverable by projection-based methods; distributed weak signals require the SOM's locality-based aggregation. By Proposition~\ref{prop:disparity}(c), the correlations in Table~\ref{tab:leakage} translate into a lower bound of $L \cdot |r|\sigma_\zeta / \sqrt{p(1-p)}$ on the worst-case group-disparate output of any $L$-Lipschitz downstream consumer.

A natural concern is whether SOMtime's advantage reflects the SOM's higher parameter count rather than its topological inductive bias. We evaluated autoencoders at three capacity scales (supplemental material). The large autoencoder ($1.4M$ parameters) with a 22-dimensional bottleneck achieves near-perfect reconstruction ($\mathrm{MSE} = 0.001$) yet attains a maximum per-axis Spearman correlation of only $0.34$ (Table~\ref{tab:leakage}) and a rotation-invariant Spearman of $0.42$ (Table~\ref{tab:leakage}) on WVS Canada---both well below SOMtime's $0.85$. Critically, increasing autoencoder capacity does not improve sensitive-attribute recovery proportionally: capacity buys reconstruction fidelity but not the geometric organization that makes $s$ recoverable along a low-dimensional axis. This indicates that the phenomenon we observe is specific to the SOM's competitive learning dynamics and fixed lattice, which concentrate distributed sensitive signals into geometric structure---exactly as Proposition~\ref{prop:som} predicts.

We evaluated autoencoders at three scales (supplement). The large AE ($1.4M$ parameters, 22-dim bottleneck) achieves near-perfect reconstruction ($\mathrm{MSE} = 0.001$) yet reaches only $0.34$ per-axis and $0.42$ rotation-invariant Spearman on WVS Canada---well below SOMtime's $0.85$. Capacity buys reconstruction fidelity but not the geometric organization that makes $s$ recoverable along a low-dimensional axis, isolating the phenomenon to the SOM's competitive learning dynamics as Proposition~\ref{prop:som} predicts.

\noindent\textbf{Topology preservation alone does not explain SOMtime's advantage.} A second concern is whether \emph{any} topology-preserving manifold learner would reproduce SOMtime's signature, generalizing the contribution beyond SOMs. Isomap addresses this: it is competitive with PCA, UMAP, and AE on WVS, reaching rotation-invariant Spearman $0.43$ on CA and $0.37$ on DE, but does not approach SOMtime on any dataset. On RO, Age, and Gains, Isomap's rotation-invariant Spearman is essentially zero, while SOMtime's $z$-axis Spearman is $0.59$, $0.83$, and $0.43$. Topology preservation in the geodesic-MDS sense is therefore insufficient; what produces structured leakage is the combination of competitive prototype learning, fixed lattice topology, and the named activation-energy axis that aggregates distributed weak signals into an interpretable direction. Isomap, like PCA, depends on capturing variance structure (along geodesics rather than principal directions) and is therefore vulnerable to the same suppression of low-variance distributed signal that Lemma~\ref{lem:pca} formalizes; the SOM's competitive learning does not weight directions by variance.

\noindent\textbf{Topology preservation alone does not explain SOMtime's advantage.} Isomap tests whether any topology-preserving learner would reproduce SOMtime's signature. It reaches rotation-invariant Spearman $0.43$ on CA and $0.37$ on DE (competitive with PCA/UMAP/AE), but is essentially zero on RO, Age, and Gains where SOMtime's $z$-axis Spearman is $0.59$, $0.83$, and $0.43$. Topology preservation in the geodesic-MDS sense is insufficient. Isomap, like PCA, captures variance structure (along geodesics rather than principal directions) and is thus vulnerable to the low-variance-distributed-signal suppression Lemma~\ref{lem:pca} formalizes. The SOM's competitive learning does not weight directions by variance.

We did not extend the trajectory recovery procedure to baseline embeddings. The rotation-invariant correlations in Table~\ref{tab:leakage} establish that baselines lack a linearly recoverable monotone direction at the level SOMtime exposes; a rotation-invariant trajectory analysis remains an appropriate extension for capturing potential nonlinear orderings the multiple-correlation metric would miss.

\begin{figure}[t]
    \centering
    \includegraphics[width=\linewidth, trim= 0 0 0 2.5em, clip]{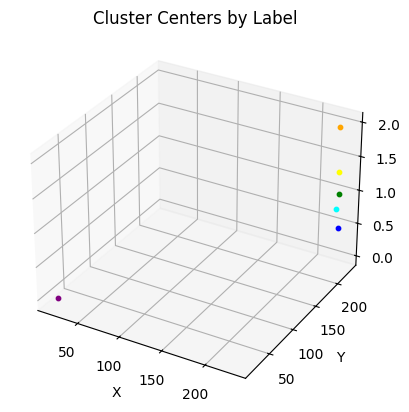}
    \caption{Sensitive attribute emergence in SOMtime. Age groups form a monotonic ordering along the learned surface. Colors: purple, blue, green, yellow, orange. Cyan: NA.}
    \label{fig:centers}

\end{figure}

\subsection*{Global Ordering Emergence}

Figure~\ref{fig:centers} visualizes the SOM embedding for WVS Canada. Age groups form a clear monotonic ordering along the SOM's latent surface; cyan markers (non-respondents) are randomly distributed near the center, consistent with absence of age-related signal. On WVS Canada, SOMtime's graph-diameter path achieves $\rho = 0.85$ with age; PC1 of each baseline embedding attains substantially weaker correlation across all five WVS countries and Census. The SOM's topology-preserving competitive learning with high lattice resolution amplifies the distributed signal of age across correlated survey responses into coherent geometric structure.

%WVS SOMs were trained country-by-country; recovery of age, income, and capital gains from Census used a single SOM. Pairwise correlations between Census attributes are low (age-income $-0.11$, income-gains $0.07$, age-gains $0.15$), showing multiple sensitive attributes can be recovered independently. Capital Gains shows the algorithm does not recover all sensitive attributes equally.

WVS SOMs were trained country-by-country; a single SOM recovers age, income, and gains on Census. Low pairwise correlations between Census attributes (age-income $-0.11$, income-gains $0.07$, age-gains $0.15$) show independent recovery of multiple sensitive attributes; Capital Gains shows the algorithm does not recover them equally.

\subsection*{Recovering Sensitive Topology from 3D Embedding}

Table~\ref{tab:top} reports accuracy in recovering sensitive attribute groups from non-sensitive data. Accuracy here refers to whether a true edge $A \to B$ (e.g., age group $25{-}35$ followed by $35{-}45$) is detected by the trajectory recovery procedure (supplementary). SOMtime achieves 100\% recovery of age for multiple WVS countries, 0.82 for Census age, and 0.76 for Census income---far above the chance rate of $1/5! \approx 0.008$ for 5 ordered groups. This empirical recovery realises Corollary~\ref{cor:dp}: clustering an embedding with structured leakage produces axis-aligned group-disparate partitions.

\begin{table}[t]
\small
\centering
\caption{Accuracy of recovered topology ordering per domain. WVS: we omit centroids denoting non-respondents.}
\label{tab:top}
\begin{tabular}{llc}
\toprule
\textbf{Dataset} & \textbf{Domain} & \textbf{Recovery Accuracy} \\
\midrule
\multirow{5}{*}{WVS}
& Canada  & 1.00 \\
& Romania & 1.00 \\
& Germany & 1.00 \\
& China   & 0.60 \\
& USA     & 0.80 \\
\midrule
\multirow{3}{*}{Census (KDD)}
& Age            & 0.82 \\
& Income         & 0.76 \\
& Capital Gains  & 0.22 \\
\bottomrule
\end{tabular}

\end{table}

\section{Discussion}

We frame SOMtime as an \emph{auditing tool} that makes structured leakage visible, not as a fairness solution. Probing classifiers and SOMtime are complementary: probes assess worst-case extractability with label access; SOMtime assesses whether $s$ is a dominant organizing principle without label access.

\subsection*{Structured Leakage and Fairness Violation Are Equivalent in Unsupervised Pipelines}

%The classical critique of fairness through unawareness in supervised settings is that a model trained without $s$ can still learn $s$-correlated decision boundaries via proxy features~\citep{dwork2012fairness,DBLP:journals/corr/KleinbergMR16,10.1145/3457607}. Our finding is the unsupervised analogue: a topology-preserving representation encodes $s$ as a dominant organizing axis, and any downstream consumer inherits $s$-dependence. The distinction sometimes drawn between privacy and fairness concerns~\citep{LIU2026108034,Rabonato_Berton_2024}---privacy concerning recoverability of $s$, fairness concerning the dependence of outcomes on $s$---makes our claim precise: structured leakage is not merely recoverability; it is \emph{organization}. Distance-, density-, and partition-based operations interact with that organization mechanically, producing $s$-dependent outcomes whether or not a probe is ever trained.
Our finding is the unsupervised analogue of the classical critique of fairness through unawareness~\citep{dwork2012fairness,DBLP:journals/corr/KleinbergMR16,10.1145/3457607}: a topology-preserving representation encodes $s$ as a dominant organizing axis, and any downstream consumer inherits $s$-dependence. The privacy--fairness distinction~\citep{LIU2026108034,Rabonato_Berton_2024} sharpens the claim: structured leakage is not recoverability, it is \emph{organization}. Distance-, density-, and partition-based operations interact with that organization mechanically, producing $s$-dependent outcomes whether or not a probe is ever trained.

The implication for practitioners: auditing for structured leakage is auditing for fairness, in the same way that monitoring proxy-feature correlations is fairness auditing in the supervised setting. Probing classifiers remain useful for adversarial worst-case extractability, but do not detect the ambient $s$-dependence that all downstream consumers of a leaking embedding will produce by default. SOMtime fills this gap.

\subsection{Limitations and Ethical Considerations}

%We do not propose a fairness mitigation. Three concrete directions for future work follow naturally: routine representation auditing combining probing-classifier tests~\citep{iwasawa2018censoring,madras2018learning} with global ordering checks; fairness-aware SOM regularization penalizing monotonic correlation with sensitive attributes or enforcing demographic entropy on neighborhoods, analogous to fairlet constraints~\citep{chierichetti2017fair}; and topology-aware lattice constraints balancing neuron occupancy across groups, extending fair dimensionality reduction~\citep{samadi2018price,peltonen2023fair} to the discrete setting. We focus on ordinal sensitive attributes (age, income) where monotonic correlation is a natural measure; categorical attributes (race, gender) would require different metrics. Datasets are tabular; extension to image, text, or graph embeddings is future work. The scaling rule $K = 5 \cdot N^{0.54}$ keeps the SOM tractable for $N$ up to $\approx 27{,}000$; scalability beyond this remains untested.

We do not propose a fairness mitigation. Three directions follow naturally: routine auditing combining probing-classifier tests~\citep{iwasawa2018censoring,madras2018learning} with global ordering checks; fairness-aware SOM regularization~\citep{chierichetti2017fair}; and topology-aware lattice constraints balancing neuron occupancy across groups~\citep{samadi2018price,peltonen2023fair}. We focus on ordinal attributes (age, income); categorical attributes (race, gender) would require different metrics. Datasets are tabular; extension to image, text, or graph embeddings, and scalability beyond $N \approx 27{,}000$, remain future work.

SOMtime is designed as an auditing tool, but the same machinery can recover demographic structure from data with sensitive attributes deliberately withheld---the inverse use. This strengthens the case that fairness through unawareness is a weak defense and that responsible deployment requires explicit auditing. Organizations should constrain SOMtime to audit-only workflows producing summary leakage measures for governance, with standard data-protection controls (access logging, role-based authorization, retention limits). The methods are unlikely to expand profiling capability beyond what trained probes achieve on the same embeddings, but the geometric visibility makes the risk salient enough to justify explicit governance.

\section{Conclusion}

Unsupervised learning is not neutral. Purely unsupervised representations can encode withheld sensitive attributes as emergent geometric structure; topology-preserving methods like Self-Organizing Maps are particularly effective at organizing this structure along a single interpretable axis; and structured leakage at correlation level $r$ implies a quantitative lower bound on the worst-case group-disparate output of any Lipschitz consumer. On two real-world datasets, SOMtime exposed monotonic orderings aligned with age and income along a single named direction, while PCA, UMAP, t-SNE, AE, and Isomap---evaluated under the rotation-invariant maximum---attained at most half the Spearman SOMtime did along its theoretically named axis. The Isomap result is significant: a topology-preserving manifold learner is not sufficient to reproduce SOMtime's signature, isolating the SOM's competitive-learning prototype updates and named activation-energy axis as the operative mechanism.

These findings carry a clear implication: auditing for sensitive-attribute leakage must extend beyond probing classifiers and beyond supervised predictors to the unsupervised representations that increasingly serve as the foundation of modern ML pipelines.

\section{GenAI Disclosure}
LLMs were used to assist with editing and grammar of text. All ideas, concepts, results, and writing/running of code were done by humans alone.

\section{Trajectory Recovery Procedure}
\label{ap_alg}

This procedure is referenced from the Method section of the main paper. Given centroids of the SOM units in 3D embedding coordinates (BMU $(x,y)$ position plus activation $z$), Algorithm~\ref{algo_traj} greedily constructs a directed adjacency matrix encoding monotone progression along the activation axis. Each unit is connected to its nearest neighbour with strictly higher activation, with normalization in the $(x,y)$ dimensions and scaling of the activation dimension by the lattice size $K$ so that adjacency reflects local proximity on the lattice while activation differences dominate at the global scale.

\begin{algorithm}[h]
\small
\caption{Trajectory Adjacency Matrix Construction}
\label{algo_traj}
\begin{algorithmic}[1]
\REQUIRE $\mathbf{C} \in \mathbb{R}^{n \times 3}$: Centroids (x, y, z = activation)
\REQUIRE $K \in \mathbb{R}$: Dimension of SOM
\ENSURE $\mathbf{A} \in \{0,1\}^{n \times n}$: Adjacency matrix
\STATE $\mathbf{A} \gets \mathbf{0}_{n \times n}$
\FOR{$i = 1$ to $n$}
    \IF{$\sum_{k} A_{k,i} > 0$ \textbf{or} $\sum_{k} A_{i,k} > 0$}
        \STATE \textbf{continue}
    \ENDIF
    \STATE $d_{\min} \gets \infty;\ j^* \gets -1$
    \FOR{$j = 1$ to $n$}
        \IF{$i = j$}
            \STATE \textbf{continue}
        \ENDIF
        \STATE $\tilde{\mathbf{c}}_i \gets \left( \tfrac{C_{i,1}}{\sum_{k} C_{k,1}}, \tfrac{C_{i,2}}{\sum_{k} C_{k,2}}, K \cdot C_{i,3} \right)$
        \STATE $\tilde{\mathbf{c}}_j \gets \left( \tfrac{C_{j,1}}{\sum_{k} C_{k,1}}, \tfrac{C_{j,2}}{\sum_{k} C_{k,2}}, K \cdot C_{j,3} \right)$
        \STATE $d \gets \|\tilde{\mathbf{c}}_i - \tilde{\mathbf{c}}_j\|_2$
        \IF{$d < d_{\min}$ \textbf{and} $C_{i,3} < C_{j,3}$}
            \STATE $j^* \gets j;\ d_{\min} \gets d$
        \ENDIF
    \ENDFOR
    \IF{$j^* \neq -1$}
        \STATE $A_{i,j^*} \gets 1$
    \ENDIF
\ENDFOR
\STATE \textbf{return} $\mathbf{A}$
\end{algorithmic}
\end{algorithm}

\section{Ablation Study of AE Size}
\label{ap_a}

This appendix presents the autoencoder capacity ablation referenced in Section~V (Results) of the main paper. The large autoencoder with 1.4M parameters and 22-dimensional bottleneck achieves near-perfect reconstruction (MSE $= 0.001$) but does not approach SOMtime's per-axis Spearman correlation, demonstrating that capacity alone is not sufficient to reproduce SOMtime's auditing signature.

\begin{table*}[t]
\caption{Autoencoder capacity ablation across all datasets (age as withheld sensitive attribute for WVS; age and income for Census). Maximum per-axis $|\rho|$ Spearman correlation. Reconstruction MSE on standardized features shown in parentheses.}
\label{tab:ae_ablation}
\centering
\begin{tabular}{lrc ccccc cc}
\toprule
& & & \multicolumn{5}{c}{\textbf{WVS (Age)}} & \multicolumn{2}{c}{\textbf{Census (KDD)}} \\
\cmidrule(lr){4-8} \cmidrule(lr){9-10}
\textbf{Model} & \textbf{Params} & \textbf{Bneck}
& \textbf{Canada} & \textbf{Romania} & \textbf{Germany} & \textbf{China} & \textbf{USA}
& \textbf{Age} & \textbf{Income} \\
\midrule
\multicolumn{10}{l}{\textit{Small autoencoders ($\sim$7--9K params)}} \\
AE-small  & 7K  & 3d  & 0.29 \; {\scriptsize(0.446)} & 0.09 \; {\scriptsize(0.419)} & 0.28 \; {\scriptsize(0.434)} & 0.25 \; {\scriptsize(0.372)} & 0.25 \; {\scriptsize(0.469)} & 0.20 \; {\scriptsize(0.168)} & 0.17 \; {\scriptsize(0.168)} \\
AE-small  & 8K  & 10d & 0.19 \; {\scriptsize(0.152)} & 0.08 \; {\scriptsize(0.164)} & 0.24 \; {\scriptsize(0.148)} & 0.23 \; {\scriptsize(0.116)} & 0.24 \; {\scriptsize(0.163)} & 0.17 \; {\scriptsize(0.009)} & 0.13 \; {\scriptsize(0.009)} \\
AE-small  & 9K  & 22d & 0.28 \; {\scriptsize(0.031)} & 0.13 \; {\scriptsize(0.022)} & 0.22 \; {\scriptsize(0.063)} & 0.25 \; {\scriptsize(0.023)} & 0.20 \; {\scriptsize(0.054)} & 0.20 \; {\scriptsize(0.001)} & 0.18 \; {\scriptsize(0.001)} \\
\midrule
\multicolumn{10}{l}{\textit{Medium autoencoders ($\sim$91--97K params)}} \\
AE-medium & 94K & 3d  & 0.22 \; {\scriptsize(0.304)} & 0.10 \; {\scriptsize(0.262)} & 0.23 \; {\scriptsize(0.297)} & 0.19 \; {\scriptsize(0.246)} & 0.29 \; {\scriptsize(0.342)} & 0.18 \; {\scriptsize(0.073)} & 0.15 \; {\scriptsize(0.073)} \\
AE-medium & 95K & 10d & 0.21 \; {\scriptsize(0.065)} & 0.08 \; {\scriptsize(0.047)} & 0.28 \; {\scriptsize(0.047)} & 0.23 \; {\scriptsize(0.039)} & 0.22 \; {\scriptsize(0.075)} & 0.19 \; {\scriptsize(0.001)} & 0.08 \; {\scriptsize(0.001)} \\
AE-medium & 97K & 22d & 0.23 \; {\scriptsize(0.004)} & 0.09 \; {\scriptsize(0.005)} & 0.28 \; {\scriptsize(0.014)} & 0.21 \; {\scriptsize(0.006)} & 0.25 \; {\scriptsize(0.013)} & 0.21 \; {\scriptsize(0.001)} & 0.22 \; {\scriptsize(0.001)} \\
\midrule
\multicolumn{10}{l}{\textit{Large autoencoders ($\sim$1.4M params)}} \\
AE-large  & 1.4M & 3d  & 0.19 \; {\scriptsize(0.124)} & 0.05 \; {\scriptsize(0.075)} & 0.25 \; {\scriptsize(0.108)} & 0.23 \; {\scriptsize(0.105)} & 0.23 \; {\scriptsize(0.150)} & 0.22 \; {\scriptsize(0.014)} & 0.14 \; {\scriptsize(0.014)} \\
AE-large  & 1.4M & 10d & 0.24 \; {\scriptsize(0.002)} & 0.11 \; {\scriptsize(0.002)} & 0.26 \; {\scriptsize(0.003)} & 0.19 \; {\scriptsize(0.002)} & 0.26 \; {\scriptsize(0.002)} & 0.20 \; {\scriptsize(0.0002)} & 0.25 \; {\scriptsize(0.0002)} \\
AE-large  & 1.4M & 22d & 0.34 \; {\scriptsize(0.001)} & 0.15 \; {\scriptsize(0.001)} & 0.29 \; {\scriptsize(0.002)} & 0.23 \; {\scriptsize(0.002)} & 0.27 \; {\scriptsize(0.002)} & 0.21 \; {\scriptsize(0.0001)} & 0.17 \; {\scriptsize(0.0001)} \\
\midrule
\multicolumn{3}{l}{\textbf{SOMtime}} & \textbf{0.85} & \textbf{0.59} & \textbf{0.73} & \textbf{0.58} & \textbf{0.52} & \textbf{0.83} & \textbf{0.69} \\
\multicolumn{3}{l}{SOM params}       & 4.28M          & 3.35M          & 1.51M          & 3.16M          & 2.66M          & 2.05M          & 2.05M          \\
\bottomrule
\multicolumn{10}{l}{\footnotesize Values shown as $|\rho|_{\text{Spearman}}$ \; {\scriptsize(Recon.\ MSE)}. All $p$-values $< 0.01$.} \\
\end{tabular}
\end{table*}

\section{Hyperparameter Sensitivity}
\label{ap_b}

The SOM hyperparameters used throughout the main experiments---lattice dimension $K = 5 \cdot N^{0.54}$, neighborhood width $\sigma = 0.7$, and learning rate $0.75$---were fixed before any sensitive-attribute correlation was computed and are not tuned using sensitive-attribute labels. To verify that these choices do not coincide with isolated sensitive sweet spots, we report sensitivity sweeps for the two most consequential hyperparameters on WVS Canada (age as the withheld sensitive attribute), holding all other settings fixed.

\subsection*{Lattice Size}

The lattice dimension $K = 5 \cdot N^{0.54}$ follows a slight modification of the heuristic of \citet{vesanto2000clustering}, which balances representational capacity against computational cost. Table~\ref{tab:k_ablation} reports sensitivity to the exponent of the scaling rule across a fine grid from $N^{0.40}$ to $N^{0.60}$.

\begin{table}[h]
\caption{SOMtime sensitivity to the lattice-size exponent on WVS Canada (age withheld). Spearman $\rho$ between the activation $z$-axis and the withheld age attribute is reported for nine values of the exponent in the scaling rule $K = 5 \cdot N^{c}$, with $c \in [0.40, 0.60]$.}
\label{tab:k_ablation}
\centering
\begin{tabular}{lc}
\toprule
\textbf{Scaling} & \textbf{Spearman $\rho$} \\
\midrule
$5 \cdot N^{0.40}$ & 0.70 \\
$5 \cdot N^{0.45}$ & 0.80 \\
$5 \cdot N^{0.50}$ & 0.81 \\
$5 \cdot N^{0.51}$ & 0.81 \\
$5 \cdot N^{0.52}$ & 0.82 \\
$5 \cdot N^{0.53}$ & 0.84 \\
$\mathbf{5 \cdot N^{0.54}}$ \textit{(used in paper)} & \textbf{0.85} \\
$5 \cdot N^{0.55}$ & 0.85 \\
$5 \cdot N^{0.60}$ & 0.84 \\
\bottomrule
\end{tabular}
\end{table}

Three patterns are visible. First, the Spearman rises smoothly and monotonically with the exponent from $0.40$ ($\rho = 0.70$) to $0.54$ ($\rho = 0.85$), with no local extrema along the climb. Second, the chosen exponent of $0.54$ marks the precise beginning of a plateau: $0.54$ and $0.55$ both attain $\rho = 0.85$, and the value at $0.60$ ($\rho = 0.84$) is within $0.01$ of the peak. Third, the smoothness of the curve means no value in the explored range produces sensitive-attribute structure that deviates materially from the rest: the Spearman varies within a $0.15$ window across the entire $5\times$ range of effective lattice resolutions implied by $c \in [0.40, 0.60]$. This is the sensitivity profile predicted by Proposition~2 of the main paper: smaller exponents correspond to coarser lattice resolution $\epsilon$, loosening the $L_s \epsilon$ preservation bound and allowing sensitive-attribute structure to fall through the resolution of the SOM. Increasing the exponent past $0.54$ tightens $\epsilon$ further but does not improve preservation, because the additional resolution captures structure orthogonal to the $s$-aligned direction rather than refining it. The Vesanto--Alhoniemi heuristic therefore selects the minimum lattice resolution sufficient to recover the structured leakage signal on this data, with no apparent gains from further scaling.

\subsection*{Learning Rate}

The learning rate $\alpha = 0.75$ controls how aggressively prototype vectors are updated toward each presented input during competitive learning. Table~\ref{tab:lr_ablation} reports sensitivity across four values spanning low, moderate, and high settings.

\begin{table}[h]
\caption{SOMtime sensitivity to learning rate on WVS Canada (age withheld). Spearman $\rho$ between the activation $z$-axis and the withheld age attribute is reported for four learning-rate values.}
\label{tab:lr_ablation}
\centering
\begin{tabular}{lc}
\toprule
\textbf{Learning rate} $\alpha$ & \textbf{Spearman $\rho$} \\
\midrule
0.25 & 0.62 \\
0.50 & 0.77 \\
$\mathbf{0.75}$ \textit{(used in paper)} & \textbf{0.85} \\
1.00 & 0.73 \\
\bottomrule
\end{tabular}
\end{table}

The learning rate exhibits a unimodal sensitivity curve with the peak at $0.75$ and graceful degradation in both directions. This shape has a clear mechanistic interpretation. At low learning rates ($\alpha = 0.25$), prototype updates are too small for the SOM to fully discover the data manifold during training; the lattice under-fits, distributed sensitive signals are averaged out across neurons rather than concentrated along the activation axis, and the per-axis Spearman drops to $0.62$. At high learning rates ($\alpha = 1.00$), late-stage updates can overwrite the topology established by earlier samples, breaking the topology-preservation property on which Proposition~2 of the main paper depends; the lattice retains capacity but loses the locally consistent ordering needed for $s$ to align with the activation axis, and the per-axis Spearman drops to $0.73$. The degradation is asymmetric---steeper on the high side ($-0.12$) than on the low side ($-0.08$ at the corresponding distance)---which is consistent with the topology-preservation mechanism: insufficient learning degrades performance gradually, while excessive learning can actively destroy lattice structure.

\subsection*{Summary}

Both sensitivity curves confirm that the SOMtime hyperparameters used throughout the paper sit at the principled choice from prior literature ($K$) or near a clear local optimum ($\alpha$), rather than at a brittle artefact-generating configuration. The Spearman correlation degrades smoothly under hyperparameter perturbation, and the magnitudes of degradation under the explored range remain comfortably above any baseline result in Table~1 of the main paper: the lowest SOMtime Spearman observed in any ablation cell, $0.62$, still exceeds the highest rotation-invariant Spearman of any baseline ($0.44$ for PCA on WVS Canada) by a comfortable margin. SOMtime's advantage over the baselines is therefore not contingent on hyperparameter tuning to a specific configuration.

\bibliography{Fairness}

\end{document}